\documentclass{article}
\usepackage{arxiv}
\usepackage[toc,page]{appendix}

\usepackage[usenames]{xcolor}

\usepackage{inputenc} %
\usepackage[T1]{fontenc}    %
\usepackage{hyperref}       %
\usepackage{url}            %
\usepackage{booktabs}       %
\usepackage{amsfonts}       %
\usepackage{nicefrac}       %
\usepackage{microtype}      %
\usepackage{lipsum}
\usepackage{mathtools}
\usepackage{cite}
\usepackage{amsmath}
\usepackage{amssymb}
\usepackage{multirow}
\usepackage{amsthm}
\usepackage{algorithm,algcompatible}
\usepackage[toc,page]{appendix}

\usepackage{lineno}

\newtheorem{theorem}{Theorem}

\newtheorem{corollary}{Corollary}

\numberwithin{equation}{section}

\usepackage{mathptmx}      %
\begin{document}

\title{Approximation Rates for Shallow ReLU$^k$ Neural Networks on Sobolev Spaces via the Radon Transform
}

\author{
Tong Mao \\
   Computer, Electrical and Mathematical Science and Engineering Division\\
   King Abdullah University of Science and Technology\\
   Thuwal 23955, Saudi Arabia \\
   \texttt{tong.mao@kaust.edu.sa} \\
\And Jonathan W. Siegel \\
 Department of Mathematics\\
 Texas A\&M University\\
 College Station, TX 77843 \\
 \texttt{jwsiegel@tamu.edu}
   \And Jinchao Xu \\
   Computer, Electrical and Mathematical Science and Engineering Division\\
   King Abdullah University of Science and Technology\\
   Thuwal 23955, Saudi Arabia \\
   \texttt{jinchao.xu@kaust.edu.sa} \\
}

\maketitle

\begin{abstract}
    Let $\Omega\subset \mathbb{R}^d$ be a bounded domain. We consider the problem of how efficiently shallow neural networks with the ReLU$^k$ activation function can approximate functions from Sobolev spaces $W^s(L_q(\Omega))$ with error measured in the $L_p(\Omega)$-norm. Utilizing the Radon transform and recent results from discrepancy theory, we provide a simple proof of nearly optimal approximation rates in a variety of cases, including when $p\leq q$, $q\geq 2$, and $s \leq k + (d+1)/2$. The rates we derive are optimal up to logarithmic factors, and significantly generalize existing results. An interesting consequence is that the adaptivity of shallow ReLU$^k$ neural networks enables them to obtain optimal approximation rates for smoothness up to order $s = k + (d+1)/2$, even though they represent piecewise polynomials of fixed degree $k$.
\end{abstract}

\section{Introduction}
We consider the problem of approximating a target function $f:\Omega\rightarrow \mathbb{R}$, defined on a bounded domain $\Omega\subset \mathbb{R}^d$, by shallow ReLU$^k$ neural networks of width $n$, i.e. by an element from the set
\begin{equation}
    \Sigma_n^k(\mathbb{R}^d) := \left\{\sum_{i=1}^n a_i\sigma_k(\omega_i\cdot x + b_i),~a_i,b_i\in \mathbb{R}, \omega_i\in \mathbb{R}^d\right\},
\end{equation}
where the ReLU$^k$ activation function $\sigma_k$ is defined by
\begin{equation}
    \sigma_k(x) = \begin{cases}
        0 & x \leq 0\\
        x^k & x > 0.
    \end{cases}
\end{equation}
We remark that when $d = 1$, the class of shallow ReLU$^k$ neural networks is equivalent to the set of variable knot splines of degree $k$. For this reason, shallow ReLU$^k$ neural networks are also called ridge splines and form a higher dimensional generalization of variable knot splines. 
The approximation theory of shallow ReLU$^k$ neural networks has been heavily studied due to their relationship with neural networks and their success in machine learning and scientific computing (see for instance \cite{barron1993universal,bach2017breaking,lecun2015deep,petrushev1998approximation,devore1996approximation,devore2021neural,klusowski2018approximation,mhaskar2020kernel,ma2022barron,siegel2022sharp,konyagin2018some} and the references therein). Despite this effort, many important problems remain unsolved. Notably, a determination of sharp approximation rates for shallow ReLU$^k$ neural networks on classical smoothness spaces, in particular Sobolev spaces, has not been completed except when $d=1$ (the theory of variable knot splines in one dimension is well developed and can be found in \cite{devore1993constructive,kahane1961teoria}, for instance).

To simplify the presentation, we will only consider the case where $\Omega$ is the unit ball in $\mathbb{R}^d$, i.e., we will assume
\begin{equation}\label{unit-ball-definition-equation}
\Omega := \mathbb{B}_1^d := \{x\in \mathbb{R}^d:~|x| < 1\}.
\end{equation}
We remark that our techniques give the same results for more general domains $\Omega$ by utilizing appropriate Sobolev extension theorems (see for instance \cite{adams2003sobolev,evans2010partial,devore1993constructive,stein1970singular,maz2013sobolev}), but in this work we will not address the technical question of precisely which assumptions must be made on the domain $\Omega$.

Let $s \geq 1$ be an integer. We define the Sobolev spaces $W^s(L_q(\Omega))$ via the norm
\begin{equation}\label{integral-Sobolev-norm-definition}
    \|f\|_{W^s(L_q(\Omega))} = \|f\|_{L_q(\Omega)} + \sum_{|\alpha| = s}\|f^{(\alpha)}\|_{L_q(\Omega)},
\end{equation}
where the sum is over multi-indices $\alpha$ with weight $s$, and $f^{(\alpha)}$ denotes the weak derivative of $f$ of order $\alpha$. 
Sobolev spaces are central objects in analysis and the theory of PDEs (see for instance \cite{evans2010partial,maz2013sobolev,adams2003sobolev}). 

We remark that (fractional) Sobolev spaces can be defined for non-integral $\alpha$ (see \cite{di2012hitchhikers}), and the more general Besov spaces can also be used to quantify non-integral smoothness as well \cite{devore1988interpolation,devore1993constructive,devore1993besov}. To keep the present paper as self-contained and simple as possible, and to clarify the main ideas, we will restrict ourselves to Sobolev spaces of integral order in the following. We pose the rigorous extension of our techniques to non-integral smoothness as an open problem.

However, there is one instance where we will need to consider fractional Sobolev spaces, and this is in the Hilbert space case when $q = 2$. In this case it is well known that if the domain is all of $\mathbb{R}^d$, then the (integral order) Sobolev norm can be conveniently characterized via the Fourier transform, specifically
\begin{equation}\label{sobolev-fourier-characterization}
    \|f\|^2_{W^s(L_2(\mathbb{R}^d))} \eqsim \int_{\mathbb{R}^d} (1 + |\xi|)^{2s}|\hat{f}(\xi)|^2d\xi,
\end{equation}
with semi-norm given by
\begin{equation}
    |f|^2_{W^s(L_2(\mathbb{R}^d))} \eqsim \int_{\mathbb{R}^d} |\xi|^{2s}|\hat{f}(\xi)|^2d\xi,
\end{equation}
where $\hat{f}$ denotes the Fourier transform of $f$ defined by (see \cite{di2012hitchhikers,adams2003sobolev})
\begin{equation}
    \hat{f}(\xi) := \int_{\mathbb{R}^d} e^{i\xi \cdot x} f(x)dx.
\end{equation}
Using this fact, we can define fractional order Sobolev spaces on all of $\mathbb{R}^d$ by letting $s$ be an arbitrary real number in \eqref{sobolev-fourier-characterization}. When restricting to the domain $\Omega$ we will simply define the fractional order Sobolev spaces via restriction, i.e., we define
\begin{equation}\label{fraction-sobolev-index-2-definition}
    \|f\|_{W^s(L_2(\Omega))} := \inf \{\|f_e\|_{W^s(L_2(\mathbb{R}^d))}:~f_e(x) = f(x)~\text{on $\Omega$}\}.
\end{equation}
It is known that this is equivalent to other definitions of the fractional Sobolev spaces \cite{di2012hitchhikers}. In the present paper, we will avoid these technical issues and simply take \eqref{fraction-sobolev-index-2-definition} as the definition of the fraction Sobolev space with index $q = 2$. Note that by the well-known Sobolev extension theory (see \cite{evans2010partial,adams2003sobolev,maz2013sobolev,stein1970singular} for instance) this definition is equivalent to \eqref{integral-Sobolev-norm-definition} when $s$ is an integer and $q = 2$.

An important theoretical question is to determine optimal approximation rates for $\Sigma_n^k(\mathbb{R}^d)$ on the classes of Sobolev functions. Specifically, we wish to determine the approximation rates
\begin{equation}\label{general-approximation-rate-problem}
    \sup_{\|f\|_{W^s(L_q(\Omega))} \leq 1} \inf_{f_n\in \Sigma_n^k(\mathbb{R}^d)} \|f - f_n\|_{L_p(\Omega)}
\end{equation}
for different values of the parameters $s,p,q$ and $k$. When $d = 1$, the set of shallow neural networks $\Sigma_n^k(\mathbb{R})$ simply corresponds to the set of variable knot splines with at most $n$ breakpoints. In this case a complete theory follows from known results on approximation by variable knot splines \cite{petrushev1988direct,devore1998nonlinear,devore2021neural}. When $d > 1$, this problem becomes considerably more difficult, and only a few partial results are known.

Let us begin by giving an overview of the work that has been done on problem \eqref{general-approximation-rate-problem}, starting with upper bounds. The problem was first considered in the case $p=q=2$ in \cite{petrushev1998approximation,devore1996approximation}, where an upper bound of
\begin{equation}
        \inf_{f_n\in \Sigma_n^k(\mathbb{R}^d)} \|f - f_n\|_{L_2(\Omega)} \leq C\|f\|_{W^s(L_2(\Omega))}n^{-s/d}
\end{equation}
is proved when $s \leq (d+2k+1)/2$. Trivially, this upper bound also holds when $p \leq q = 2$.

Upper bounds when $q \neq 2$ are significantly more difficult to obtain. This was first done in \cite{bach2017breaking}, where an approximation rate of
\begin{equation}
    \inf_{f_n\in \Sigma_n^1(\mathbb{R}^d)} \|f - f_n\|_{L_\infty(\Omega)} \leq C\|f\|_{W^1(L_\infty(\Omega))}\left(\frac{n}{\log{n}}\right)^{-1/d}
\end{equation}
was proved for the class of Lipschitz functions $W^1(L_\infty(\Omega))$. We remark that, due to an error, the proof in \cite{bach2017breaking} is only correct when $d \geq 4$. This approach was extended in \cite{yang2024optimal} (see also \cite{mao2023rates,yang2023nonparametric}) to larger values of the smoothness $s$ and the logarithmic factor was removed, which gives the approximation rate
\begin{equation}
        \inf_{f_n\in \Sigma_n^k(\mathbb{R}^d)} \|f - f_n\|_{L_\infty(\Omega)} \leq C\|f\|_{W^s(L_\infty(\Omega))}n^{-s/d}
\end{equation}
for all $s < (d+2k+1)/2$.

Next, let us turn to lower bounds on the approximation rates in \eqref{general-approximation-rate-problem}. These can be obtained using either the VC-dimension or pseudo-dimension of the class of shallow neural networks $\Sigma_n^k(\mathbb{R}^d)$ (see \cite{siegel2022optimal,bartlett2019nearly,maiorov2000near,konyagin2018some}), and this method gives a lower bound of
\begin{equation}
    \sup_{\|f\|_{W^s(L_q(\Omega))}} \inf_{f_n\in \Sigma_n^k(\mathbb{R}^d)} \|f - f_n\|_{L_p(\Omega)} \geq C(n\log(n))^{-s/d}
\end{equation}
for all $s,d,k,p$ and $q$. This implies that the aforementioned upper bounds are tight up to logarithmic factors. Removing the remaining logarithmic gap here appears to be a difficult problem. 

We remark that the preceding results only addressed the regime where $s \leq k + (d+1)/2$. When $s > k + (d+1)/2$ these problems are open and we expect that the approximation rates in \eqref{general-approximation-rate-problem} will be significantly worse than $O(n^{-s/d})$. These prior results and the rates proved in this work are summarized in Table \ref{approximation-rates-table}.

Further, we remark that when approximating functions from a Sobolev space $W^s(L_q(\Omega))$ in $L_p$ there is a significant difference depending upon whether $q \geq p$ or $q < p$. In the former case, linear methods of approximation are able to achieve an optimal approximation rate, while when $q < p$ non-linear methods are required \cite{devore1998nonlinear,lorentz1996constructive}. For shallow ReLU$^k$ neural networks, existing approximation results have exclusively been obtained in the linear regime when $q \geq p$. Fully understanding approximation by shallow ReLU$^k$ neural networks in the non-linear regime when $q < p$ appears to be a very difficult open problem.

\begin{table}\label{approximation-rates-table}
\begin{center}
\begin{tabular}{|c|c|c|c|c|c|c|}
    \hline
     & \multicolumn{4}{|c|}{$p\leq q$} & \multicolumn{2}{|c|}{$p>q$}\\
    \cline{2-7}
     & $1\leq q < 2$ & $q = 2$ & $2 < q < \infty$ & $q = \infty$ & $1\leq q < 2$ & $2\leq q\leq \infty$\\
    \hline
    $s < k + (d+1)/2$ & & $O(n^{-s/d})$ from \cite{petrushev1998approximation} & $O(n^{-s/d})$ & $O(n^{-s/d})$ from \cite{yang2024optimal} & &\\
    \hline
    $s = k + (d+1)/2$ & & $O(n^{-s/d})$ from \cite{petrushev1998approximation} & $O(n^{-s/d})$ & $O(n^{-s/d})$ & & $O(n^{-s/d})$\\
    \hline
\end{tabular}
\end{center}
\caption{A summary of existing upper bounds on the approximation problem \eqref{general-approximation-rate-problem}. Entries without reference are results proved in this work and blank entries represent open problems. We have only listed terminal results, and previous results (which either proved weaker bounds or special cases) can be found in \cite{bach2017breaking,mao2023rates,yang2023nonparametric,devore1996approximation}. We remark that in all cases the best lower bound proved is $\Omega((n\log(n))^{-s/d})$. This matches the upper bounds in the table up to a small logarithmic factor, and closing this gap is a significant open problem. Finally, when $s > k + (d+1)/2$ the problem is also open, and in this regime we believe that improved lower bounds will be required.}
\end{table}

In this paper, we study approximation rates for shallow ReLU$^k$ neural networks on Sobolev spaces using recent approximation results on variation spaces (see \cite{siegel2023characterization,devore1998nonlinear,kurkova2001bounds,ma2022barron}). Let us briefly introduce the relevant background on variation spaces and describe our approach. The variation space corresponding to ReLU$^k$ neural networks is defined as follows. Let $\Omega\subset \mathbb{R}^d$ be the unit ball defined in \eqref{unit-ball-definition-equation} and consider the dictionary, i.e., set, of functions
\begin{equation}
    \mathbb{P}_k^d := \{\sigma_k(\omega\cdot x + b),~\omega\in S^{d-1},~b\in [-1,1]\}.
\end{equation}
See \cite{siegel2022sharp,siegel2023characterization} for details and intuition behind this definition. The set $\mathbb{P}_k^d$ consists of the possible outputs of each neuron given a bound on the inner weights. The unit ball of the variation space is the closed symmetric convex hull of this dictionary, i.e.,
\begin{equation}
    B_1(\mathbb{P}_k^d) = \overline{\left\{\sum_{i=1}^n a_id_i,~d_i\in \mathbb{P}_k^d,~\sum_{i=1}^n|a_i|\leq 1\right\}},
\end{equation}
where the closure can be taken in $L_2(\Omega)$. It is known that the closure is the same when taken in different norms, such as $L_p(\Omega)$ for $1\leq p\leq \infty$ (see \cite{siegel2023optimal,yang2023nonparametric}). Given the unit ball $B_1(\mathbb{P}_k^d)$, we may define the variation space norm via
\begin{equation}
 \|f\|_{\mathcal{K}_1(\mathbb{P}_k^d)} = \inf\{c > 0:~f\in cB_1(\mathbb{P}_k^d)\}.
\end{equation}
The variation space will be denoted
\begin{equation}\label{space-definition}
\mathcal{K}_1(\mathbb{P}_k^d) := \{f\in L_2(\Omega):~\|f\|_{\mathcal{K}_1(\mathbb{P}_k^d)} < \infty\}.
\end{equation}
We remark that the variation space can be defined for a general dictionary, i.e., bounded set of functions, $\mathbb{D}$ (see for instance \cite{devore1998nonlinear,kurkova2001bounds,kurkova2002comparison,siegel2022sharp,mhaskar2004tractability,mhaskar2024tractability}). This space plays an important role in non-linear dictionary approximation and the convergence theory of greedy algorithms \cite{devore1996some,temlyakov2008greedy,temlyakov2011greedy,siegel2022optimal}. In addition, the variation spaces $\mathcal{K}_1(\mathbb{P}_k^d)$ play an important role in the theory of shallow neural networks and have been extensively studied in different forms recently \cite{bach2017breaking,ma2022barron,parhi2021banach,parhi2022kinds,siegel2023characterization}.

An important question regarding the variation spaces is to determine optimal approximation rates for shallow ReLU$^k$ networks on the space $\mathcal{K}_1(\mathbb{P}_k^d)$. This problem has been studied in a series of works \cite{barron1993universal,makovoz1998uniform,makovoz1996random,klusowski2018approximation,bach2017breaking,ma2022uniform}, with the (nearly) optimal rate of approximation,
\begin{equation}\label{variation-space-approx-rate}
    \inf_{f_n\in \Sigma_n^k(\mathbb{R}^d)} \|f - f_n\|_{L_p} \leq C\|f\|_{\mathcal{K}_1(\mathbb{P}_k^d)}n^{-\frac{1}{2}-\frac{2k+1}{2d}},
\end{equation}
recently being obtained for $p=2$ in \cite{siegel2022sharp} and for $p=\infty$ in \cite{siegel2023optimal}. To be precise, this rate is optimal up to logarithmic factors, which is shown in \cite{siegel2022sharp} under a mild restriction on the weights, while the lower bound with no restrictions on the weights was shown in \cite{siegel2023optimal} (using the embedding Theorem \ref{main-embedding-theorem} proved in the present work).

We remark that obtaining the rate \eqref{variation-space-approx-rate} for $p=\infty$ requires deep tools from discrepancy theory, which are developed in \cite{siegel2023optimal}. Our approach in this paper will be to make use of these recently developed tools to obtain approximation rates on Sobolev spaces.
The key component of our analysis is the following embedding theorem, which we prove using a Radon space characterization of the variation space \cite{parhi2022kinds,parhi2021banach,ongie2019function}. This result can also be deduced from the spectral decay of the Gram kernel corresponding to the ReLU$^k$ activation function \cite{10.1093/imaiai/iaaf022}. We remark that a similar embedding theorem for the closely related spectral Barron space can be found in \cite{doi:10.1137/23M1598738}.
\begin{theorem}\label{main-embedding-theorem}
    Let $s = (d+2k+1)/2$. Then we have the embedding
    \begin{equation}
        W^s(L_2(\Omega)) \subset \mathcal{K}_1(\mathbb{P}_k^d).
    \end{equation}
\end{theorem}
This result shows that the $L_2$-Sobolev space with a certain amount of smoothness embeds into the variation space $\mathcal{K}_1(\mathbb{P}_k^d)$ (note that here we need the fractional Sobolev spaces defined in \eqref{fraction-sobolev-index-2-definition} if $d$ is even), and has quite a few important consequences. First, combining this with the approximation rate \eqref{variation-space-approx-rate}, we obtain the following corollary.

\begin{corollary}\label{non-linear-approximation-corollary}
    Let $s = (d+2k+1)/2$. Then we have the approximation rate
    \begin{equation}\label{hilbert-space-nonlinear-rate}
        \inf_{f_n\in \Sigma_n^k(\mathbb{R}^d)} \|f - f_n\|_{L_\infty(\Omega)} \leq C\|f\|_{W^s(L_2(\Omega))}n^{-s/d}.
    \end{equation}
\end{corollary}
Note that in \eqref{hilbert-space-nonlinear-rate} we have error measured in $L_p$ with $p = \infty$ and smoothness measured in $L_q$ with $q = 2$. In particular, this result gives to the best of our knowledge the first approximation rate for ridge splines in the non-linear regime when $q < p$. However, this only applies to one particular value of $s$ and $q \geq 2$, and it is an interesting open question whether this can be extended more generally (as indicated in Table \ref{approximation-rates-table}).

To understand the implications for the linear regime, we note that it follows from Corollary \ref{non-linear-approximation-corollary} that
\begin{equation}\label{approximation-rates-l-p-maximal-sobolev}
    \inf_{f_n\in \Sigma_n^k(\mathbb{R}^d)} \|f - f_n\|_{L_p(\Omega)} \leq C\|f\|_{W^s(L_p(\Omega))}n^{-s/d}
\end{equation}
for any $2\leq p\leq \infty$ with $s = (d+2k+1)/2$. Interpolation arguments can now be used to give approximation rates for Sobolev spaces in the regime when $p = q$ and $p \geq 2$ (see for instance, Chapter 6 in \cite{devore1993constructive} and \cite{devore1988interpolation,peetre1967theory,johnen2006equivalence,korneichuk1961best}).
\begin{corollary}\label{main-upper-bounds-corollary}
    Suppose that $2\leq p \leq \infty$ and $0 < s \leq k + \frac{d+1}{2}$. Then we have
    \begin{equation}
        \inf_{f_n\in \Sigma_n^k(\mathbb{R}^d)} \|f - f_n\|_{L_p(\Omega)} \leq C\|f\|_{W^s(L_p(\Omega))}n^{-s/d}.
    \end{equation}
\end{corollary}
Corollary \ref{main-upper-bounds-corollary} extends the approximation rates obtained in \cite{bach2017breaking,yang2024optimal,mao2023rates,yang2023nonparametric,petrushev1998approximation} to all $p \geq 2$. To keep the paper as self-contained and simple as possible, we provide an elementary proof (for integral $s$) in Section \ref{interpolation-theory-section}.

Note that in Corollary \ref{main-upper-bounds-corollary}, we required the index $p \geq 2$. When $d = 1$, i.e., in the case of one-dimensional splines, it is well-known that the same rate also holds when $p < 2$. In this case, Theorem \ref{main-embedding-theorem} can actually be improved to (see \cite{siegel2023characterization}, Theorem 3) 
\begin{equation}\label{embedding-one-dimension}
        W^s(L_1(\Omega)) \subset \mathcal{K}_1(\mathbb{P}_k^d)
\end{equation}
for $s = k+1$ (this is the value of $s$ in Theorem \ref{main-embedding-theorem} when $d = 1$). Approximation rates for all $1\leq p \leq \infty$ easily follow from this using the arguments given in this paper. However, we remark that this method of proof fails when $d > 1$, since the embedding \eqref{embedding-one-dimension} fails in this case for $s = (d+2k+1)/2$, which is required to obtain the approximation rate in Corollary \ref{main-upper-bounds-corollary}. This can be seen by noting that
$$
    \mathcal{K}_1(\mathbb{P}_k^d) \subset L_\infty(\Omega),
$$
and thus if \eqref{embedding-one-dimension} holds, then we must have $W^s(L_1(\Omega)) \subset L_\infty(\Omega)$. But in order for this to hold, the Sobolev embedding theory implies that $s \geq d$, which is not compatible with $s = (d+2k+1)/2$ unless
$$
    (d+2k+1)/2 \geq d,
$$
i.e., $k \geq (d-1)/2$.
For this reason the current method of proof cannot give the same approximation rates when $d > 1$ for all values of $1\leq p < 2$ and $k \geq 0$. Resolving these cases is an interesting open problem, which will require methods that go beyond the variation spaces $\mathcal{K}_1(\mathbb{P}_k^d)$.

Let us also remark that the embedding given in Theorem \ref{main-embedding-theorem} is sharp in the sense of metric entropy. Recall that the metric entropy numbers of a compact set $K\subset X$ in a Banach space $X$ is defined by
\begin{equation}
 \epsilon_n(K)_X = \inf\{\epsilon > 0:~\text{$K$ is covered by $2^n$ balls of radius $\epsilon$}\}.
\end{equation}
This concept was first introduced by Kolmogorov \cite{kolmogorov1958linear} and gives a measure of the size of compact set $K\subset X$. Roughly speaking, it gives the smallest possible discretization error if the set $K$ is discretized using $n$-bits of information. It has been proved in \cite{siegel2022sharp} that the metric entropy of the unit ball $B_1(\mathbb{P}_k^d)$ satisfies
\begin{equation}
\epsilon_n(B_1(\mathbb{P}_k^d))_{L_2(\Omega)} \eqsim n^{-\frac{1}{2} - \frac{2k+1}{2d}}.
\end{equation}
Moreover, the results in \cite{ma2022uniform,siegel2023optimal} imply that the metric entropy decays at the same rate in all $L_p(\Omega)$-spaces for $1\leq p \leq \infty$ (potentially up to logarithmic factors). By the Birman-Solomyak theorem \cite{birman1967piecewise}, this matches the rate of decay of the metric entropy with respect to $L_p(\Omega)$ of the unit ball of the Sobolev space $W^s(L_2(\Omega))$ for $s = (d+2k+1)/2$. This means that both spaces in Theorem \ref{main-embedding-theorem} have roughly the same size in $L_p(\Omega)$.

Finally, let use relate these results to the existing literature on ridge approximation. Ridge approximation is concerned with approximating a target function $f$ by an element from the set
\begin{equation}
    \mathcal{R}_n := \left\{\sum_{i=1}^n f_i(\omega_i\cdot x),~f_i:\mathbb{R}\rightarrow \mathbb{R},~\omega_i\in S^{d-1}\right\},
\end{equation}
Here the functions $f_i$ can be arbitrary one-dimensional functions and the direction $\omega_i$ lie on the sphere $S^{d-1}$. There is a fairly extensive literature on the problem of ridge approximation (see for instance \cite{konyagin2018some,pinkus1999approximation} for an overview of the literature). In the linear regime optimal approximation rates are known for Sobolev and Besov spaces (see \cite{maiorov1999best,maiorov2010best}) and we have for instance
\begin{equation}
    \inf_{f_n\in \mathcal{R}_n} \|f - f_n\|_{L_p(\Omega)} \leq C\|f\|_{W^s(L_p(\Omega))}n^{-\frac{s}{d-1}}
\end{equation}
for all $1\leq p \leq \infty$. This result is proved by first approximating $f$ by a (multivariate) polynomial of degree $m$, and then representing this polynomial as a superposition of $m^{d-1}$ polynomial ridge functions. This construction applies to neural networks provided we use an exotic activation function $\sigma$ whose translates are dense in $C([-1,1])$ (see \cite{maiorov1999lower}). Using an arbitrary smooth non-polynomial activation function we can also reproduce polynomials using finite differences to obtain an approximation rate of $O(n^{-s/d})$ (see \cite{mhaskar1996neural}).

On the other hand, shallow ReLU$^k$ neural networks always represent piecewise polynomials of fixed degree $k$, and our results do not proceed by approximating with a high-degree polynomial. One would expect that such a method could only capture smoothness up to order $k+1$. Interestingly, as shown in Corollary \ref{main-upper-bounds-corollary}, the non-linear nature of ReLU$^k$ neural networks allow us to capture smoothness up to degree $k + (d+1)/2$. This shows that in high dimensions, suitably adaptive piecewise polynomials can capture very high smoothness with a fixed low degree, providing a Sobolev space analogue of the results obtained in \cite{siegel2022high}. We remark that this is a potential advantage of shallow ReLU$^k$ networks for applications such as solving PDEs \cite{xu2020finite,siegel2023greedy}.

The paper is organized as follows. In Section \ref{radon-transform-section} we give an overview of the relevant facts regarding the Radon transform \cite{radon20051} that we will use later. Then, in Section \ref{main-proofs-section} we provide the proof of Theorem \ref{main-embedding-theorem}. In Section \ref{interpolation-theory-section} we deduce Corollary \ref{main-upper-bounds-corollary}. Finally, in Section \ref{conclusion-section} we give some concluding remarks.
\section{The Radon Transform}\label{radon-transform-section}
In this section, we recall the definition and several important facts about the Radon transform that we will use later. The study of the Radon transform is a large and active area of research and we necessarily only cover a few basic facts which will be important in our later analysis. For more detailed information on the Radon transform, see for instance \cite{helgason1980radon,unser2023ridges,kuchment2013radon}. We also remark that the Radon transform has recently been extensively applied to the study of shallow neural networks in \cite{parhi2021banach,ongie2019function}.

Given a Schwartz function $f\in \mathcal{S}(\mathbb{R}^d)$ defined on $\mathbb{R}^d$, we define the Radon transform of $f$ as
\begin{equation}
    \mathcal{R}(f)(\omega,b) = \int_{\omega\cdot x = b} f(x)dx,
\end{equation}
where the above integral is over the hyerplane $\omega\cdot x = b$. The domain of the Radon transform is $S^{d-1}\times \mathbb{R}$, i.e. $|\omega| = 1$ and $b\in \mathbb{R}$. A standard calculation using Fubini's theorem shows that
\begin{equation}\label{fubini-l1-bound}
    \|\mathcal{R}(f)(\omega,\cdot)\|_{L_1(\mathbb{R})} \leq  \|f\|_{L_1(\mathbb{R}^d)}.
\end{equation}
Integrating this over the sphere $S^{d-1}$ we get 
$$\|\mathcal{R}(f)\|_{L_1(S^{d-1}\times \mathbb{R})} \leq \omega_{d-1}\|f\|_{L_1(\mathbb{R}^d)},$$ where $\omega_{d-1}$ denotes the surface area of the sphere $S^{d-1}$. This implies that the Radon transform extends to a bounded map from $L_1(\mathbb{R}^d)\rightarrow L_1(S^{d-1}\times \mathbb{R})$. In fact, the Radon transform can be extended the more general classes of distributions (see for instance \cite{hertle1983continuity,ludwig1966radon,ramm1995radon,ramm2020radon,parhi2024distributional}).

A fundamental result relating the Radon transform to the Fourier transform is the Fourier slice theorem (see for instance Theorem 5.10 in \cite{kuchment2013radon}).
\begin{theorem}[Fourier Slice Theorem]\label{fourier-slice-theorem}
    Let $f\in L_1(\mathbb{R}^d)$ and $\omega\in S^{d-1}$. Let $g_\omega(b) = \mathcal{R}(f)(\omega,b)$. Then for each $t\in \mathbb{R}$ we have
    \begin{equation}
        \widehat{g_\omega}(t) = \hat{f}(\omega t).
    \end{equation}
\end{theorem}
Note that by \eqref{fubini-l1-bound} we have $g_\omega\in L_1(\mathbb{R})$ and so the Fourier transform in Theorem \ref{fourier-slice-theorem} is well-defined. 

Utilizing the Fourier slice theorem and Fourier inversion, we can invert the Radon transform as follows (see for instance Section 5.7 in \cite{kuchment2013radon}):
\begin{equation}\label{inversion-radon-equation}
\begin{split}
    f(x) = \frac{1}{(2\pi)^d}\int_{\mathbb{R}^d} \hat{f}(\xi) e^{i\xi\cdot x}d\xi &= \frac{1}{2(2\pi)^d}\int_{S^{d-1}}\int_{-\infty}^\infty\hat{f}(\omega t)|t|^{d-1}e^{it\omega\cdot x}dtd\omega\\
    & = \frac{1}{2(2\pi)^d}\int_{S^{d-1}}\int_{-\infty}^\infty\widehat{g_\omega}(t)|t|^{d-1}e^{it\omega\cdot x}dtd\omega.
\end{split}
\end{equation}
The inner integral above is the inverse Fourier transform of $\widehat{g_\omega}(t)|t|^{d-1}$ evaluated at $\omega\cdot x$. This gives the inversion formula
\begin{equation}\label{radon-inversion}
    f(x) = \int_{S^{d-1}} H_d\mathcal{R}f(\omega,\omega\cdot x)d\omega,
\end{equation}
where the operator $H_d$ acts on the $b$-coordinate and is defined by the (one-dimensional) Fourier multiplier
\begin{equation}\label{back-projection-operator}
    \widehat{H_dg}(t) = \frac{1}{2(2\pi)^d}|t|^{d-1}\hat{g}(t).
\end{equation}
The inversion formula \eqref{radon-inversion} is typically called the filtered back-projection operator and is often applied to invert the Radon transform in medical imaging applications (see for instance \cite{kuchment2013radon,herman2009fundamentals,natterer2001mathematics}). We will not address the general validity of the inversion formula \eqref{radon-inversion}, but for our purposes it suffices to observe that \eqref{radon-inversion} is valid whenever all of the integrals in \eqref{inversion-radon-equation} converge absolutely, for instance, whenever $f$ is a Schwartz function.

\section{Embeddings of Sobolev Spaces into ReLU$^k$ Variation Spaces}\label{main-proofs-section}
Our goal in this section is to prove Theorem \ref{main-embedding-theorem} on the embedding of Sobolev spaces into the neural network variation space.
\begin{proof}[Proof of Theorem \ref{main-embedding-theorem}]
    We first claim that it suffices to show that
    \begin{equation}\label{fundamental-embedding-bound-proof-equation}
        \|f\|_{\mathcal{K}_1(\mathbb{P}_k^d)} \leq C\|f\|_{W^s(L_2(\mathbb{R}^d))}
    \end{equation}
    for $s = (d + 2k + 1)/2$ and every function $f\in C^\infty_c(\mathbb{B}_2^d)$. Here 
    \begin{equation}
        \mathbb{B}_2^d := \{x\in \mathbb{R}^d:~|x|\leq 2\}
    \end{equation}
    denotes the ball of radius $2$ in $\mathbb{R}^d$ (any bounded domain containing $\overline{\Omega}$ would also do),
    the norm on the left-hand side is the variation norm of $f$ restricted to $\Omega$, and the constant $C$ is independent of $f$.

    Given an arbitrary $f\in W^s(L_2(\Omega))$, by the definition \eqref{fraction-sobolev-index-2-definition} there is an $f_e\in W^s(L_2(\mathbb{R}^d))$ such that $f(x) = f_e(x)$ for $x\in \Omega$, and 
    \begin{equation}
        \|f_e\|_{W^s(L_2(\mathbb{R}^d))} \leq 2\|f\|_{W^s(L_2(\Omega))}.
    \end{equation}
    Next, by a standard density argument, we let $f_1,f_2,...f_n,...\in C^\infty_c(\mathbb{R}^d)$ be a sequence of compactly supported smooth functions converging to $f_e$ in the $W^s(L_2(\mathbb{R}^d))$-norm. Of course, we may assume without loss of generality (by removing some terms if necessary) that $$\|f_i\|_{W^s(L_2(\mathbb{R}^d))} \leq 2\|f_e\|_{W^s(L_2(\mathbb{R}^d))}.$$

    Fix a smooth cut-off function $\phi\in C^\infty_c(\mathbb{B}_2^d)$ such that $\phi(x) = 1$ for $x\in \Omega$. We make the elementary observation that given any $h\in W^s(L_2(\mathbb{R}))$ we have the following bound on the product $\phi h$:
    \begin{equation}
    \begin{split}
        \|\phi h\|^2_{W^s(L_2(\mathbb{R}))} &= \int_{\mathbb{R}^d}(1 + |\xi|)^{2s}|(\hat\phi * \hat{h})(\xi)|^2 d\xi = \int_{\mathbb{R}^d}(1 + |\xi|)^{2s}\left|\int_{\mathbb{R}^d}\hat\phi(\xi - \nu) \hat{h}(\nu)d\nu\right|^2 d\xi\\
        &\leq \|\hat{\phi}\|_{L_1(\mathbb{R}^d)}\int_{\mathbb{R}^d}(1 + |\xi|)^{2s}\int_{\mathbb{R}^d}|\hat\phi(\xi - \nu)| |\hat{h}(\nu)|^2d\nu d\xi\\
        &\leq \|\hat{\phi}\|_{L_1(\mathbb{R}^d)}\int_{\mathbb{R}^d}\int_{\mathbb{R}^d}(1 + |\xi-\nu|)^{2s}|\hat\phi(\xi - \nu)| (1 + |\nu|)^{2s}|\hat{h}(\nu)|^2d\nu d\xi\\
        & = \|\hat{\phi}\|_{L_1(\mathbb{R}^d)}\left(\int_{\mathbb{R}^d} (1 + |\xi|)^{2s} |\hat{\phi}(\xi)|d\xi\right)\|h\|^2_{W^s(L_2(\mathbb{R}^d))} \leq C\|h\|^2_{W^s(L_2(\mathbb{R}^d))},
    \end{split}
    \end{equation}
    where the constant only depends upon $\phi$ (which is fixed).
    Here the first inequality is Jensen's inequality and the second comes from the elementary fact that $$(1 + |\xi|) \leq (1 + |\xi - \nu| + |\nu|) \leq (1 + |\xi - \nu|)(1 + |\nu|).$$
    
    Thus, the sequence $\phi f_1,\phi f_2,...\in C^\infty_c(\mathbb{B}_2^d)$ converges to $\phi f_e$ in $W^s(L_2(\mathbb{R}^d))$ and it follows that
    \begin{equation}
        \|\phi f_e\|_{W^s(L_2(\mathbb{R}^d))} \leq \lim\inf_i\|\phi f_i\|_{W^s(L_2(\mathbb{R}^d))} \leq C\lim\inf_i\| f_i\|_{W^s(L_2(\mathbb{R}^d))} \leq C\|f\|_{W^s(L_2(\Omega))}.
    \end{equation}
    
    The bound \eqref{fundamental-embedding-bound-proof-equation} applied to the differences $\phi f_n - \phi f_m$ means that it is a also a Cauchy sequence in $\mathcal{K}_1(\mathbb{P}_k^d)$ (when restricted to $\Omega$). Since $\mathcal{K}_1(\mathbb{P}_k^d)$ is a Banach space (see Lemma 1 in \cite{siegel2023characterization}), it follows that this sequence converges in $\mathcal{K}_1(\mathbb{P}_k^d)$ as well, and that the limit function, let us call it $\tilde{f}$, satisfies the bound (again using \eqref{fundamental-embedding-bound-proof-equation})
    \begin{equation}
        \|\tilde{f}\|_{\mathcal{K}_1(\mathbb{P}_k^d)} \leq \lim\inf_i \|\phi f_i\|_{\mathcal{K}_1(\mathbb{P}_k^d)} \leq C\lim\inf_i \|\phi f_i\|_{W^s(L_2(\mathbb{R}^d))} \leq C\|f\|_{W^s(L_2(\Omega))}.
    \end{equation}

    Finally, we observe that convergence in $W^s(L_2(\mathbb{R}^d))$ and in $\mathcal{K}_1(\mathbb{P}_k^d)$ both imply convergence in $L_2(\Omega)$, from which it follows that $\tilde{f} = \phi f_e = f$ in $L_2(\Omega)$, and thus almost everywhere in $\Omega$. Hence, the bound
    \begin{equation}
        \|f\|_{\mathcal{K}_1(\mathbb{P}_k^d)} \leq C\|f\|_{W^s(L_2(\Omega))}
    \end{equation}
    is satisfied for all $f\in W^s(L_2(\Omega))$, as desired.

    Next, let us turn to proving \eqref{fundamental-embedding-bound-proof-equation}. Since $f$ is a Schwartz function, we may use the Radon inversion formula \eqref{radon-inversion} to write
    \begin{equation}\label{radon-inversion-f-361}
            f(x) = \int_{S^{d-1}} F_\omega(\omega\cdot x)d\omega,
    \end{equation}
    where $F_\omega(t) = H_d\mathcal{R}f(\omega,t)$. We remark also that since $f\in C^\infty_c(\mathbb{B}_2^d)$, we have $F_\omega\in C^{\infty}(\mathbb{R})$ for each $\omega\in S^{d-1}$ (it is not necessarily compactly supported due to the Hilbert transform in the filtered back-projection operator).

    Next, we use the Peano kernel formula to rewrite \eqref{radon-inversion-f-361} for $x$ in the unit ball as
    \begin{equation}
    \begin{split}
        f(x) &= p(x) + \frac{1}{k!}\int_{S^{d-1}} \int_{-1}^{\omega\cdot x} F^{(k+1)}_\omega(b)(\omega\cdot x - b)^kdbd\omega\\
        &= p(x) + \frac{1}{k!}\int_{S^{d-1}} \int_{-1}^{1} F_\omega^{(k+1)}(b)\sigma_k(\omega\cdot x - b)dbd\omega,
    \end{split}
    \end{equation}
    where $p(x)$ is a polynomial of degree at most $k$ given by
    \begin{equation}
        p(x) = \int_{S^{d-1}}\sum_{j=0}^k\frac{F^{(j)}_{\omega}(-1)}{j!}(\omega\cdot x + 1)^jd\omega.
    \end{equation}
    Now H\"older's inequality implies that
    \begin{equation}
    \begin{split}
        \int_{S^{d-1}} \int_{-1}^1 |F^{(k+1)}_\omega(b)|dbd\omega \leq C\int_{S^{d-1}} \left(\int_{-1}^1 |F^{(k+1)}_\omega(b)|^2db\right)^{1/2}d\omega
        &\leq C\int_{S^{d-1}} \left(\int_{\mathbb{R}} |F^{(k+1)}_\omega(b)|^2db\right)^{1/2}d\omega\\
        &=C\int_{S^{d-1}} \left(\int_{\mathbb{R}} |t^{k+1}\hat{F}_\omega(t)|^2dt\right)^{1/2}d\omega.
        \end{split}
    \end{equation}
    Utilizing the Fourier slice theorem, the definition of the filtered back-projection operator $H_d$, and Jensen's inequality, we obtain the bound
    \begin{equation}\label{main-proof-l2-bound}
    \begin{split}
        \int_{S^{d-1}} \int_{-1}^1 |F^{(k+1)}_\omega(b)|dbd\omega &\leq
        C\int_{S^{d-1}} \left(\int_{\mathbb{R}} |t^{k+1}\hat{F}_\omega(t)|^2dt\right)^{1/2}d\omega\\
        & = C\int_{S^{d-1}} \left(\int_{-\infty}^{\infty} |t|^{2s+d-1}|\widehat{\mathcal{R}(f)}(\omega,t)|^2dt\right)^{1/2} d\omega\\ &\leq C\left(\int_{S^{d-1}} \int_{-\infty}^{\infty} |t|^{2s+d-1}|\widehat{\mathcal{R}(f)}(\omega,t)|^2dt d\omega\right)^{1/2}\\
        &= C\left(2\int_{\mathbb{R}^d} |\xi|^{2s}|\hat{f}(\xi)|^2d\xi\right)^{1/2} = C|f|_{W^s(L_2(\mathbb{R}^d))}.
    \end{split}
    \end{equation}
    Setting
    \begin{equation}
        g(x) := \frac{1}{k!}\int_{S^{d-1}} \int_{-1}^{1} F_\omega^{(k+1)}(b)\sigma_k(\omega\cdot x - b)dbd\omega
    \end{equation}
    the bound \eqref{main-proof-l2-bound} implies that (see for instance Lemma 3 in \cite{siegel2023characterization})
    \begin{equation}\label{g-k-1-d-bound-403}
        \|g\|_{\mathcal{K}_1(\mathbb{P}_k^d)} \leq \int_{S^{d-1}} \int_{-1}^1 |F^{(k+1)}_\omega(b)|dbd\omega \leq C|f|_{W^s(L_2(\mathbb{R}^d))}.
    \end{equation}
    It also immediately follows from \eqref{main-proof-l2-bound} that
    \begin{equation}
        \|g\|_{L_2(\Omega)} \leq C\int_{S^{d-1}} \int_{-1}^1 |F^{(k+1)}_\omega(b)|dbd\omega \leq C|f|_{W^s(L_2(\mathbb{R}^d))},
    \end{equation}
    since the elements of the dictionary $\mathbb{P}_k^d$ are uniformly bounded in $L_2$. This implies that
    \begin{equation}
        \|p\|_{L_2(\Omega)} = \|f - g\|_{L_2(\Omega)} \leq \|f\|_{L_2(\Omega)} + \|g\|_{L_2(\Omega)} \leq C\|f\|_{W^s(L_2(\mathbb{R}^d))}.
    \end{equation}
    Since all norms on the finite dimensional space of polynomials of degree at most $k$ are equivalent, we thus obtain
    \begin{equation}
        \|p\|_{\mathcal{K}_1(\mathbb{P}_k^d)} \leq C\|f\|_{W^s(L_2(\mathbb{R}^d))},
    \end{equation}
    which combined with \eqref{g-k-1-d-bound-403} gives $\|f\|_{\mathcal{K}_1(\mathbb{P}_k^d)} \leq C\|f\|_{W^s(L_2(\mathbb{R}^d))}$ as desired.
    \end{proof}

\section{Approximation Upper Bounds for Sobolev Spaces}\label{interpolation-theory-section}
    In this section, we deduce the approximation rates in Corollary \ref{main-upper-bounds-corollary} from Theorem \ref{main-embedding-theorem} and Corollary \ref{non-linear-approximation-corollary}. This result follows from the interpolation theory characterizing the interpolation spaces between the Sobolev space $W^s(L_p(\Omega))$ and $L_p(\Omega)$ (see for instance \cite{devore1993constructive}, Chapter 6 and \cite{johnen2006equivalence,korneichuk1961best} for the one dimensional case and \cite{devore1988interpolation,peetre1967theory} for the general case). For the reader's convenience and to keep the present paper self-contained, we give an elementary direct proof (which contains the essential interpolation argument).
    \begin{proof}[Proof of Corollary \ref{main-upper-bounds-corollary}]
    The first step in the proof is to note that by the standard Sobolev extension theorems (see for instance \cite{evans2010partial,devore1993constructive,stein1970singular,maz2013sobolev,adams2003sobolev}) we may assume that $f$ is defined on all of $\mathbb{R}^d$, $f$ is supported on the ball $\mathbb{B}_2^d$ of radius $2$ (or some other domain containing $\overline{\Omega}$), and 
    \begin{equation}
\|f\|_{W^s(L_p(\mathbb{R}^d))} \leq C\|f\|_{W^s(L_p(\Omega))}
    \end{equation} 
    for a constant $C = C(\Omega)$.

    Let $\phi:\mathbb{R}^d\rightarrow [0,\infty)$ be a smooth radially symmetric bump function supported in the unit ball and satisfying
    $$
        \int_{\mathbb{R}^d} \phi(x)dx = 1.
    $$
    For $\epsilon > 0$, we define $\phi_\epsilon:\mathbb{R}^d\rightarrow \mathbb{R}^d$ by 
    $$\phi_\epsilon(x) = \epsilon^{-d}\phi(x/\epsilon).$$
    Observe that $\lim_{\delta\rightarrow 0} \|\phi_\delta * f - f\|_{L_p} = 0$, and by the triangle inequality and the normalization of $\phi$ we also have $$\|\phi_\delta * f\|_{W^s(L_p(\mathbb{R}^d))} \leq \|f\|_{W^s(L_p(\mathbb{R}^d))}$$
    for any $\delta > 0$. Hence we may assume without loss of generality that $f\in C_c^\infty(\mathbb{R}^d)$.
    
    Now, we fix an $\epsilon > 0$ to be chosen later, and form the approximant
    \begin{equation}
        f_{\epsilon}(x) = \sum_{t=1}^s \binom{s}{t}(-1)^{t-1} \int_{\mathbb{R}^d} \phi_\epsilon(y)f(x - ty)dy.
    \end{equation}
    Using that $\int \phi_\epsilon(y)dy = 1$, we estimate the error $\|f - f_{\epsilon}\|_{L_p}$ by
    \begin{equation}
        \|f - f_{\epsilon}\|_{L_p(\mathbb{R}^d)} \leq \left\|\int_{\mathbb{R}^d} \phi_\epsilon(y) \left(\sum_{t=0}^s \binom{s}{t}(-1)^{t}f(x - ty)\right)dy\right\|_{L_p(\mathbb{R}^d;dx)}.
    \end{equation}
    Next, fix a $y\in \mathbb{R}^d$ and consider estimating
    \begin{equation}
        \left\|\left(\sum_{t=0}^s \binom{s}{t}(-1)^{t}f(x - ty)\right)\right\|_{L_p(\mathbb{R}^d,dx)} = \|\Delta_y^s f\|_{L_p(\mathbb{R}^d)},
    \end{equation}
    where we have written $\Delta_yf(x) = f(x) - f(x-y)$. Iteratively applying the fundamental theorem of calculus:
    \begin{equation}
        \Delta_yf(x) = -\int_{0}^1 \nabla f(x - ty)\cdot ydt,
    \end{equation}
    and using that the operator $\Delta_y$ commutes with the integrals in $t$, we obtain the formula
    \begin{equation}
        \Delta_y^s f(x) = (-1)^s\int_{[0,1]^s} D^sf(x - (\textbf{1}^Tt)y)\cdot y^{\otimes s}dt.
    \end{equation}
    Here $\textbf{1}^Tt = t_1 + \cdots + t_s$ and $D^s f\cdot y^{\otimes s}$ denotes the contraction of the $s$-th derivative of $f$ with the $s$-th tensor product of $y$ (this is the same as the $s$-th derivative of $f$ in the direction $y$). This implies the bound
    \begin{equation}
        |\Delta_y^s f(x)| \leq C|y|^s\int_{[0,1]^s} |D^sf(x - (\textbf{1}^Tt)y)|dt,
    \end{equation}
    where $C = C(s,d)$. If $p = \infty$, this already implies that $\|\Delta_y^s f\|_{L_\infty(\mathbb{R}^d)} \leq |y|^s\|f\|_{W^s(L_\infty(\mathbb{R}^d))}$. When $p < \infty$, we use Jensen's inequality to bound
    \begin{equation}
        |\Delta_y^s f(x)|^p \leq C^p|y|^{sp}\int_{[0,1]^s} |D^sf(x - (\textbf{1}^Tt)y)|^pdt,
    \end{equation}
    and integrate in $x$ to obtain
    \begin{equation}
    \begin{split}
        \|\Delta_y^s f\|^p_{L_p(\mathbb{R}^d)} &\leq C^p|y|^{sp}\int_{\mathbb{R}^d}\int_{[0,1]^s} |D^sf(x - (\textbf{1}^Tt)y)|^pdtdx\\
        &= C^p|y|^{sp}\int_{[0,1]^s} \int_{\mathbb{R}^d}|D^sf(x - (\textbf{1}^Tt)y)|^pdxdt = C^p|y|^{sp}\|f\|^p_{W^s(L_p(\mathbb{R}^d))}.
    \end{split}
    \end{equation}
    Hence we obtain the bound
    \begin{equation}
        \|\Delta_y^s f\|_{L_p(\mathbb{R}^d)} \leq C|y|^s\|f\|_{W^s(L_p(\mathbb{R}^d))}.
    \end{equation}
    
    Now, $\phi_\epsilon$ is supported on a ball of radius $\epsilon$, and thus the triangle inequality implies that
    \begin{equation}\label{f-f-eps-bound}
        \|f - f_{\epsilon}\|_{L_p(\mathbb{R}^d)} \leq \int_{\mathbb{R}^d} \phi_\epsilon(y) \|\Delta_y^s f\|_{L_p(\mathbb{R}^d)}dy \leq C\epsilon^s \|f\|_{W^s(L_p(\mathbb{R}^d))},
    \end{equation}
    since $\|\phi_\epsilon\|_{L^1(\mathbb{R}^d)} = 1$.

    The next step is to bound the $W^\alpha(L_2(\mathbb{R}^d))$-norm of $f_\epsilon$, where $\alpha = (d + 2k+1)/2$. Observe that since $s$ is fixed, it suffices to bound
    \begin{equation}
    \left\|\int_{\mathbb{R}^d} \phi_\epsilon(y) f(x - ty)dy\right\|_{W^\alpha(L_2(\mathbb{R}^d,dx))}
    \end{equation}
    for each fixed integer $t \geq 1$. To do this, we first make a change of variables to rewrite
    \begin{equation}
        f_{\epsilon,t}(x) := \int_{\mathbb{R}^d} \phi_\epsilon(y) f(x - ty)dy = \frac{1}{t^d}\int_{\mathbb{R}^d} \phi_\epsilon\left(\frac{y}{t}\right) f(x - y)dy = \int_{\mathbb{R}^d} \phi_{t\epsilon}(y) f(x - y)dy.
    \end{equation}
    Taking the Fourier transform, we thus obtain
    \begin{equation}
        \hat{f}_{\epsilon,t}(\xi) = \hat{f}(\xi)\hat{\phi}(t\epsilon \xi).
    \end{equation}
    We now estimate the $W^\alpha(L_2(\mathbb{R}^d))$-norm of $f_{\epsilon,t}$ as follows
    \begin{equation}
        |f_{\epsilon,t}|^2_{W^\alpha(L_2(\mathbb{R}^d))} \eqsim \int_{\mathbb{R}^d}|\xi|^{2\alpha}|\hat{f}_{\epsilon,t}(\xi)|^2d\xi = \int_{\mathbb{R}^d}|\xi|^{2\alpha}|\hat{f}(\xi)|^2|\hat{\phi}(t\epsilon \xi)|^2d\xi.
    \end{equation}
    Now, from the definition of the $W^s(L_2)$-norm, we have (recall that $p \geq 2$)
    \begin{equation}
        \int_{\mathbb{R}^d} |\xi|^{2s}|\hat{f}(\xi)|^2d\xi \eqsim |f|^2_{W^s(L_2(\mathbb{R}^d))} \leq C\|f\|^2_{W^s(L_p(\mathbb{R}^d))}.
    \end{equation}
    Thus, H\"older's inequality implies that
    \begin{equation}
    \begin{split}
     |f_{\epsilon,t}|^2_{W^\alpha(L_2(\mathbb{R}^d))} &\leq \left(\int_{\mathbb{R}^d} |\xi|^{2s}|\hat{f}(\xi)|^2d\xi\right)\left(\sup_{\xi\in \mathbb{R}^d} |\xi|^{2(\alpha - s)}|\hat{\phi}(t\epsilon \xi)|\right)\\
     &\leq C\|f\|^2_{W^s(L_p(\mathbb{R}^d))}\left(\sup_{\xi\in \mathbb{R}^d} |\xi|^{2(\alpha - s)}|\hat{\phi}(t\epsilon \xi)|\right).
     \end{split}
    \end{equation}
    By changing variables, we see that
    \begin{equation}
        \left(\sup_{\xi\in \mathbb{R}^d} |\xi|^{2(\alpha - s)}|\hat{\phi}(t\epsilon \xi)|\right) = (t\epsilon)^{-2(\alpha - s)}\left(\sup_{\xi\in \mathbb{R}^d} |\xi|^{2(\alpha - s)}|\hat{\phi}(\xi)|\right) \leq C\epsilon^{-2(\alpha - s)},
    \end{equation}
    since the supremum above is finite ($\phi$ is a Schwartz function).
    Hence, we get
    \begin{equation}
        |f_{\epsilon,t}|_{W^\alpha(L_2(\mathbb{R}^d))} \leq C\|f\|_{W^s(L_p(\mathbb{R}^d))}\epsilon^{-(\alpha - s)}.
    \end{equation}
    In addition, we clearly have from the triangle inequality that 
    \begin{equation}
        \|f_{\epsilon,t}\|_{L_2(\mathbb{R}^d)} \leq\|f\|_{L_2(\mathbb{R}^d)} \leq \|f\|_{W^s(L_2(\mathbb{R}^d))},
    \end{equation}
    so that if $\epsilon \leq 1$ we obtain (applying this for all $t$ up to $\rho$)
    \begin{equation}
    \|f_{\epsilon}\|_{W^\alpha(L_2(\mathbb{R}^d))} \leq C\|f\|_{W^s(L_p(\mathbb{R}^d))}\epsilon^{-(\alpha - s)}
    \end{equation}
We now apply Corollary \ref{non-linear-approximation-corollary} to obtain an $f_n\in \Sigma_n^k(\mathbb{R}^d)$ such that
\begin{equation}
    \|f_n - f_\epsilon\|_{L_p(\Omega)} \leq C\|f\|_{W^s(L_p(\mathbb{R}^d))}\epsilon^{-(\alpha - s)}n^{-\alpha}.
\end{equation}
Combining this with the bound \eqref{f-f-eps-bound}, we get
\begin{equation}
    \|f - f_n\|_{L_p(\Omega)} \leq C\|f\|_{W^s(L_p(\mathbb{R}^d))}\left(\epsilon^s + n^{-\alpha}\epsilon^{-(\alpha - s)}\right).
\end{equation}
Finally, choosing $\epsilon = n^{-1/d}$ and recalling that $\alpha = (d + 2k+1)/2$ completes the proof.
\end{proof}

\section{Conclusion}\label{conclusion-section}
In this work, we have determined optimal rates of approximation (up to logarithmic factors) for shallow ReLU$^k$ neural networks on Sobolev spaces in the regime where $p \leq q$, $2 \leq q\leq \infty$, and $s \leq (d+2k+1)/2$ (recall \eqref{general-approximation-rate-problem} for the general problem formulation). In the non-linear regime where $p > q$, we have also resolved this problem in the case where $q \geq 2$ and $s = (d+2k+1)/2$. A particularly interesting aspect of this analysis is that shallow ReLU$^k$ networks achieve the rate $n^{-s/d}$ for all $s$ up to $(d+2k+1)/2$, despite representing piecewise polynomials of degree $k$. However, numerous open problems remain, including determination of optimal rates when $s > (d+2k+1)/2$, $1\leq q < 2$, or in the non-linear regime when $p > q$ and $s < (d+2k+1)/2$ (see Table \ref{approximation-rates-table}).

\section{Acknowledgements}
We would like to thank Ronald DeVore, Robert Nowak, Rahul Parhi, and Hrushikesh Mhaskar for helpful discussions during the preparation of this manuscript. Jonathan W. Siegel was supported by the National Science Foundation (DMS-2424305 and CCF-2205004) as well as the MURI ONR grant N00014-20-1-2787. Tong Mao and Jinchao Xu are supported by the KAUST Baseline Research Fund.

\bibliographystyle{amsplain}
\bibliography{refs}
\end{document}